\documentclass{article}
\usepackage{graphicx}

\usepackage{amsmath}
\usepackage{amsfonts} 
\usepackage{amsthm}
\usepackage{algorithm}
\usepackage{algorithmic}
\usepackage{booktabs} 
\usepackage{caption}
\usepackage{mathtools}
\usepackage{subcaption}
\usepackage[round]{natbib}

\newcommand{\eref}[1]{Eq.~\ref{#1}}
\newcommand{\fref}[1]{Figure~\ref{#1}}
\newcommand{\tref}[1]{Table~\ref{#1}}
\newcommand{\sref}[1]{Section~\ref{#1}}
\newcommand{\aref}[1]{Algorithm~\ref{#1}}
\newcommand{\lref}[1]{Lemma~\ref{#1}}
\newcommand{\thmref}[1]{Theorem~\ref{#1}}
\newcommand{\cref}[1]{Corollary~\ref{#1}}

\DeclarePairedDelimiterX{\inp}[2]{\langle}{\rangle}{#1, #2}
\newcommand{\expnumber}[2]{{#1}\mathrm{e}{#2}}

\newtheorem{theorem}{Theorem}

\newtheorem{lemma}[theorem]{Lemma}
\newtheorem{corollary}[theorem]{Corollary}

\numberwithin{equation}{section}

\begin{document}

\title{\Large Fast Incremental SVDD Learning Algorithm with the Gaussian Kernel}

\author{
Hansi Jiang\thanks{SAS Institute Inc., 100 SAS Campus Drive, Cary, North Carolina 27513} \\
\texttt{Hansi.Jiang@sas.com}
\and
Haoyu Wang\footnotemark[1] \\
\texttt{Haoyu.Wang@sas.com}
\and
Wenhao Hu\footnotemark[1] \\
\texttt{Wenhao.Hu@sas.com}
\and
Deovrat Kakde\footnotemark[1] \\
\texttt{Dev.Kakde@sas.com}
\and
Arin Chaudhuri\footnotemark[1] \\
\texttt{Arin.Chaudhuri@sas.com}
}
\date{}

\maketitle

\begin{abstract} 
Support vector data description (SVDD) is a machine learning technique that is used for single-class classification and outlier detection. The idea of SVDD is to find a set of support vectors that defines a boundary around data. When dealing with online or large data, existing batch SVDD methods have to be rerun in each iteration. We propose an incremental learning algorithm for SVDD that uses the Gaussian kernel. This algorithm builds on the observation that all support vectors on the boundary have the same distance to the center of sphere in a higher-dimensional feature space as mapped by the Gaussian kernel function. Each iteration involves only the existing support vectors and the new data point. Moreover, the algorithm is based solely on matrix manipulations; the support vectors and their corresponding Lagrange multiplier $\alpha_i$'s are automatically selected and determined in each iteration. It can be seen that the complexity of our algorithm in each iteration is only $O(k^2)$, where $k$ is the number of support vectors. Experimental results on some real data sets indicate that FISVDD demonstrates significant gains in efficiency with almost no loss in either outlier detection accuracy or objective function value.
\end{abstract}
{\bfseries Keywords:}
Outlier detection, Classification, Support vector data description, Quadratic programming, Online learning, Internet of Things (IoT)

\section{Introduction}
Much effort has been made to detect faults and state shifts in industrial machines through monitoring data sensors. Successful fault diagnosis reduces cost of maintenance and improves both worker and machine efficiency. In machine learning, fault diagnosis can be viewed as an outlier detection problem. Support vector data description (SVDD), a machine learning technique that is used for single-class classification and outlier detection, is similar to support vector machine (SVM). SVDD was first introduced in \citet{tax2004support}, although the concept of using SVM to detect novelty was introduced in \citet{scholkopf2000support}. SVDD is used in domains where the majority of data belongs to a single class, or when one of the classes is significantly undersampled. The SVDD algorithm builds a flexible boundary around the target class data; this data boundary is characterized by observations that are designated as support vectors. Having the advantage that no assumptions about the distribution of outliers need to be made, SVDD can describe the shape of the target class without prior knowledge of the specific data distribution and can flag observations that fall outside the data boundary as potential outliers. In the case of machine monitoring, data on the normal working conditions of a machine are in abundance, whereas information from outlier system failures are few. By using SVDD on the well-sampled target class, one can obtain a boundary around the distribution of normal working data, and subsequently capture the outlier points where the machine is faulty.  

Traditional batch methods of SVDD typically pursue a global optimal solution of the SVDD problem; they suffer from low efficiency by considering all available data points. Moreover, these methods are usually ineffective when handling streaming data because the entire algorithm must be rerun with each incoming data point. In contrast, incremental methods deal with large or streaming data efficiently by focusing on smaller portions of the original optimization problem, as in \citet{syed1999incremental}. Online variants of SVDD concentrate only on the current support vector set with incoming data.

\citet{cauwenberghs2001incremental} give an incremental and decremental training algorithm for SVM. Their method, also called the C\&P algorithm, provides an exact solution for training data and one new data point. \citet{tax2003online} use a numerical method to solve incremental SVM, and they describe the relationship between incremental SVM and online SVDD. Their research was extended in \citet{laskov2006incremental}, which provides complete learning algorithms for incremental SVM and SVDD.

The algorithm given in \citet{laskov2006incremental} updates weights of each support vector based on the fact that Karush-Kuhn-Tucker (KKT) conditions must be satisfied before and after a new data point comes in. Consequently, all data points must be kept to pursue an objective value closer to the global optimal value. Furthermore, a kernel matrix must be calculated every update, which can be memory-consuming and slow for large data. 

These issues are handled by the algorithm that we propose: fast incremental support vector data description (FISVDD). One of the most important properties of support vectors is that in the most simplified form of SVDD they all have the same distance to the center of a sphere. A similar property remains even when the problem is generalized to flexible boundaries. This property is at the core of FISVDD. Unlike the method in \citet{laskov2006incremental}, FISVDD uses only matrix manipulations to find interior points and support vectors, and it is highly efficient in detecting outliers. It can be used either as a batch method or as an online method. It can be seen that the complexity of key parts of FISVDD is $O(k^2)$, where $k$ is the number of support vectors. By \citet{kakde2017peak}, the number of support vectors should be much less than the number of observations in order to avoid overfitting.

The rest of the paper is organized as follows. In \sref{svddprob}, we introduce the SVDD problem in \citet{tax2004support}. In \sref{theorems}, we state some theoretical support for FISVDD. In \sref{fisvdd}, the FISVDD algorithm is introduced and explained. In \sref{misc}, we discuss several important issues in implementing FISVDD. In \sref{experiments}, FISVDD is applied to some data sets and compared with other methods. Finally, in \sref{conclusion}, we give our conclusions.

In this paper we follow traditional linear algebra notation. Bold capital letters stand for matrices, and bold small letters stand for vectors. Specifically, matrix $\mathbf{A}$ is used as a Gaussian kernel matrix, and $\mathbf{A}_{k}$ is the Gaussian kernel matrix in the $k$th iteration. The vector $\mathbf{x}>\mathbf{0}$ stands for a positive vector, and $\mathbf{x}\ge \mathbf{0}$ stands for a nonnegative vector.

\section{The SVDD Problem}\label{svddprob}
The SVDD problem is first discussed by \citet{tax2004support}. The idea of SVDD is to find support vectors and use them to define a boundary around data. If a testing data point lies outside the boundary, it is classified as an outlier; otherwise, it is classified as normal data. The simplest form of a boundary is a sphere. For a set of data points $\mathbf{x}_1, \mathbf{x}_2, \ldots, \mathbf{x}_n$, the mathematical formulation of the problem is to find a nonnegative vector $\boldsymbol{\alpha}$ that contains Lagrange multipliers for all data points, $\|\boldsymbol{\alpha}\|_1=1$, such that the following is maximized:
\begin{equation}\label{spheredual}
L=\sum_{i=1}^{n}\alpha_{i}\inp{\mathbf{x}_{i}}{\mathbf{x}_{i}}-\sum_{i,j}\alpha_{i}\alpha _{j}\inp{\mathbf{x}_{i}}{\mathbf{x}_{j}}.
\end{equation}
Here $\inp{\mathbf{x}_{i}}{\mathbf{x}_{j}}$ is the inner product of $\mathbf{x}_i$ and $\mathbf{x}_j$. According to \citet{tax2004support}, there are three possibilities for each data point. The $\mathbf{x}_i$'s that have zero $\alpha_i$'s are \textit{interior points}. The $\mathbf{x}_i$'s for which $0<\alpha_i<C$ for a preselected $0<C\le1$ lie on the boundary and are called \textit{support vectors}. The $\mathbf{x}_i$'s for which $\alpha_i=C$ are outliers (also called \textit{bounded support vectors}, or bsv, in \citet{ben2001support}). In this paper, we assume there are no outliers in the training phase, so we set $C=1$. One example of where our algorithm would be useful is when there is a known period during which the incoming data are normal, such as streaming sensor data from machines or vehicles operating under normal conditions. Then the model can be used to detect abnormal states. To determine whether a new data point $\mathbf{z}$ lies inside the boundary, first the distance between $\mathbf{z}$ and the center of the sphere, $\mathbf{a}$, is calculated:
\begin{equation}\label{d^2orig}
d^2(\mathbf{z})=\|\mathbf{z}-\mathbf{a}\|^2 = \inp{\mathbf{z}}{\mathbf{z}} \\ 
-2\sum_i\alpha_i\inp{\mathbf{z}}{\mathbf{x}_i}+\sum_{i,j}\alpha_i\alpha_j\inp{\mathbf{x}_i}{\mathbf{x}_j}.
\end{equation}
This distance is then compared to the radius of the sphere for any support vector $\mathbf{x}_k$:
\begin{equation}\label{r^2orig}
R^2 = \inp{\mathbf{x}_k}{\mathbf{x}_k} \\ 
-2\sum_i\alpha_i\inp{\mathbf{x}_k}{\mathbf{x}_i}+\sum_{i,j}\alpha_i\alpha_j\inp{\mathbf{x}_i}{\mathbf{x}_j}.
\end{equation}
A test data point $\mathbf{z}$ is accepted if $d^2\le R^2$, and it is classified as an outlier if $d^2> R^2$. This check is also called \textit{scoring}. It is easy to derive the conclusion that scoring is equivalent to checking whether the new data point violates the current KKT conditions.  

A kernel function is needed to draw a more flexible boundary around data in order to avoid underfitting. By \citet{tax2004support}, using a kernel function is equivalent to implicitly mapping data points to a higher feature space. Usually the Gaussian kernel,
\begin{equation}\label{gaussian}
K(\mathbf{x}_i,\mathbf{x}_j) = \exp(-\frac{\|\mathbf{x}_i-\mathbf{x}_j\|_2^2}{2\sigma^2}),
\end{equation}
is preferred \citep{ben2001support,laskov2006incremental,gu2015incremental}, and the Gaussian kernel bandwidth $\sigma$ must be selected beforehand. There are some papers that discuss how to choose a proper Gaussian kernel bandwidth \citep{evangelista2007some,xiao2014two,kakde2017peak}. Throughout this paper, it is assumed that the Gaussian similarity is used and that a proper Gaussian kernel bandwidth $\sigma$ has been chosen such that the number of support vectors is much less than the number of observations. As stated in \sref{misc}, FISVDD has protections even if a bad bandwidth is provided. With the Gaussian kernel function, the objective function \eref{spheredual} can be simplified to minimizing
\begin{equation}\label{flexdual}
L=\sum_{i,j}\alpha_{i}\alpha _{j}K(\mathbf{x}_i,\mathbf{x}_j),
\end{equation}
because $K(\mathbf{x}_i,\mathbf{x}_i)=1$, $\|\boldsymbol{\alpha}\|_1=1$, and $\boldsymbol{\alpha}$ is nonnegative.

\eref{flexdual} can also be expressed in matrix form:
\begin{equation}\label{flexmat}
L=\boldsymbol{\alpha}^T\mathbf{A}\boldsymbol{\alpha},
\end{equation}
where $\mathbf{A}$ is a Gaussian similarity matrix for all support vectors and $\boldsymbol{\alpha}>\mathbf{0}$. Formulas \eref{d^2orig} and \eref{r^2orig} then become as follows, respectively:
\begin{equation}\label{d2}
d^2(\mathbf{z}) = 1-2\sum_i\alpha_iK(\mathbf{z},\mathbf{x}_i)+\sum_{i,j}\alpha_i\alpha_jK(\mathbf{x}_i,\mathbf{x}_j),
\end{equation}
\begin{equation}
R^2 = 1-2\sum_i\alpha_iK(\mathbf{x}_k,\mathbf{x}_i)+\sum_{i,j}\alpha_i\alpha_jK(\mathbf{x}_i,\mathbf{x}_j).
\end{equation}
Note that to determine whether a test data point $\mathbf{z}$ should be accepted, one can compute only
\begin{equation}\label{dtb}
Q(\mathbf{z}) = (d^2(\mathbf{z}) - R^2)/2 = \\ 
\sum_i\alpha_iK(\mathbf{x}_k,\mathbf{x}_i) - \sum_i\alpha_iK(\mathbf{z},\mathbf{x}_i).
\end{equation}
$Q(\mathbf{z})\le0$ means that $\mathbf{z}$ is an interior point. It is worth mentioning that all support vectors satisfy $d^2=R^2$, although they might have different $\alpha_i$'s.

\section{Theoretical Foundations}\label{theorems}
Here we state and prove several theorems necessary for later discussion. First, we state a lemma in \citet{smola1998learning} that a Gaussian similarity matrix has full rank. A direct conclusion of the lemma is that a Gaussian similarity matrix is symmetric positive definite (spd).
\begin{lemma}\label{spd}
Suppose $\mathbf{x}_1,\mathbf{x}_2,\ldots,\mathbf{x}_k$ are distinct points and $\sigma\neq0$. Then their Gaussian similarity matrix $\mathbf{A}$ formed with \eref{gaussian} has full rank.
\end{lemma}
\lref{spd} implies that $\mathbf{A}$ is spd and its inverse exists. Next, we state lemmas to obtain $\mathbf{A}_{k+1}^{-1}$ if $\mathbf{A}_{k}^{-1}$ is known and vice versa. In FISVDD, we need to update the inverse of the similarity matrix when a new data point comes in. The proof involves only matrix calculations and is skipped.

\begin{lemma}\label{increinv}
Suppose $\mathbf{A}_k$ and $\mathbf{A}_{k+1}$ are both Gaussian similarity matrices and
\begin{equation}\label{vvec}
\mathbf{A}_{k+1} = \begin{bmatrix}
\mathbf{A}_k && \mathbf{v}\\
\mathbf{v}^T && 1
\end{bmatrix}.
\end{equation}
If $\mathbf{A}_k^{-1}$ is known, then $\mathbf{A}_{k+1}^{-1}$
is given by
\begin{equation}\label{ak+1inv}
\mathbf{A}_{k+1}^{-1} = \begin{bmatrix}
\mathbf{A}_k^{-1} + \mathbf{p}\mathbf{p}^T/\beta && -\mathbf{p}/\beta\\
-\mathbf{p}^T/\beta && 1/\beta
\end{bmatrix},
\end{equation}
where
$\mathbf{p} = \mathbf{A}_k^{-1}\mathbf{v}$ and $\beta = 1 - \mathbf{v}^T\mathbf{A}_k^{-1}\mathbf{v}=1-\mathbf{v}^T\mathbf{p}$.
\end{lemma}

\lref{increinv} provides a method to compute $\mathbf{A}_{k+1}^{-1}$ by using $\mathbf{A}_k^{-1}$ and an incremental vector $\mathbf{v}$. Note that to compute $\mathbf{A}_{k+1}^{-1}$, we only need to compute $\mathbf{p} = \mathbf{A}_k^{-1}\mathbf{v}$. Also note that $\beta$ is the Schur complement \citep{meyer2000matrix} of $\mathbf{A}_k^{-1}$ in $\mathbf{A}_{k+1}^{-1}$. Since $\mathbf{A}_{k+1}$ is spd, $\beta$ is positive \citep{gallier2010schur}. The inverse of \lref{increinv} is straightforward and shown below. 

\begin{lemma}\label{decreinv}
Suppose $\mathbf{A}_{k+1}$ is spd and its inverse is given by
\begin{equation}
\mathbf{A}_{k+1}^{-1} = \begin{bmatrix}
\mathbf{P}_{k\times k} && \mathbf{u}\\
\mathbf{u}^T && \lambda
\end{bmatrix}.
\end{equation} 
Then the inverse of $\mathbf{A}_k$ is 
\begin{equation}\label{ak-1inv}
\mathbf{A}_{k}^{-1} = \mathbf{P} - \mathbf{u}\mathbf{u}^T/\lambda.
\end{equation}
\end{lemma}

\lref{increinv} and \lref{decreinv} together play an essential role in FISVDD to increase efficiency. It can be seen from the lemmas that only $O(k^2)$ multiplications are needed to obtain the updated matrix inverse. Next, we prove that if a positive solution is obtained for the linear system $\mathbf{A}\boldsymbol{\alpha} = \mathbf{e}$, then all data points in the system are support vectors. This is from the property that all support vectors satisfy $d^2 = R^2$.

\begin{theorem}\label{allsv}
A set of data points $\mathbf{x}_1,\mathbf{x}_2,\ldots,\mathbf{x}_k$ are all support vectors if and only if 
\begin{equation}\label{alleq}
\mathbf{A}_k\boldsymbol{\alpha} = \mathbf{e}
\end{equation}
has a positive solution, where $\mathbf{e}$ indicates a vector that contains all 1's with proper dimension.
\end{theorem}
\begin{proof}
Suppose that $\mathbf{x}_1,\mathbf{x}_2,\ldots,\mathbf{x}_k$ are all support vectors. Then they all satisfy $d^2(\mathbf{x}_i) = R^2$ in \eref{dtb}, and thus the $d^2(\mathbf{x}_i)$'s are all equal. From \eref{d2}, the middle terms, 
\begin{equation}\label{middleterm}
\sum_i\alpha_iK(\mathbf{z},\mathbf{x}_i),
\end{equation}
are all equal for any support vector $\mathbf{z}$. Putting \eref{middleterm} together for all support vectors results in the left-hand side of \eref{alleq}. Therefore, \eref{alleq} has a positive solution. On the other hand, \eref{alleq} implies that all $\mathbf{x}_i$'s satisfy $d^2(\mathbf{x}_i) = R^2$ and thus are all support vectors. 
\end{proof}

If a new data point $\mathbf{x}_{k+1}$ is added to the existing support vector set but the $(k+1)$th position in the solution to the linear system $\mathbf{A}_{k+1}\boldsymbol{\alpha}= \mathbf{e}$ is not positive, then the new data point is an interior point. This is proven in the next theorem.

\begin{theorem}\label{notallsv}
Suppose data points $\mathbf{x}_1,\mathbf{x}_2,\ldots,\mathbf{x}_k$ form a support vector set. Then a new data point $\mathbf{x}_{k+1}$ is an interior point if and only if $\mathbf{A}_{k+1}\boldsymbol{\alpha}=\mathbf{e}\Rightarrow\alpha_{k+1}\le0$. 
\end{theorem}
\begin{proof}
Suppose that $\mathbf{A}_{k+1}\boldsymbol{\alpha}=\mathbf{e}\Rightarrow\alpha_{k+1}\le0$. By \lref{increinv}, we have 
\begin{equation}
\alpha_{k+1} = [\mathbf{A}_{k+1}^{-1}\mathbf{e}]_{k+1}
=\begin{bmatrix}
-\mathbf{p}^T/\beta && 1/\beta
\end{bmatrix}\mathbf{e}.
\end{equation}
Because $\alpha_{k+1}\le0$, we have
\begin{equation}\label{alphak+1}
\alpha_{k+1} = \frac{1-\mathbf{e}^T\mathbf{A}_k^{-1}\mathbf{v}}{1-\mathbf{v}^T\mathbf{A}_k^{-1}\mathbf{v}}\le0. 
\end{equation}
Because $\beta = 1-\mathbf{v}^T\mathbf{A}_k^{-1}\mathbf{v} > 0$, we have 
\begin{equation}\label{proof1}
1-\mathbf{e}^T\mathbf{A}_k^{-1}\mathbf{v}\le0.
\end{equation} 
We want to prove that $d^2-R^2\le0$ for $\mathbf{x}_{k+1}$. Note that
\begin{equation}\label{dist2-r2}
\begin{split}
(d^2-R^2)/2 &= \boldsymbol{\alpha}_k^T{\mathbf{A}_k}_{(*i)} - \boldsymbol{\alpha}_k^T\mathbf{v} \\
&= (\mathbf{A}_k^{-1}\mathbf{e})^T{\mathbf{A}_k}_{(*i)} - (\mathbf{A}_k^{-1}\mathbf{e})^T\mathbf{v}\\
&=\mathbf{e}^T\mathbf{A}_k^{-1}{\mathbf{A}_k}_{(*i)} - \mathbf{e}^T\mathbf{A}_k^{-1}\mathbf{v} \\
&=1-\mathbf{e}^T\mathbf{A}_k^{-1}\mathbf{v},
\end{split}
\end{equation}
where ${\mathbf{A}_{k}}_{(*i)}$ is the $i$th column of ${\mathbf{A}_{k}}$. By \eref{proof1}, we have $d^2-R^2\le0$.

On the other hand, suppose $\mathbf{x}_{k+1}$ is strictly inside the boundary. Then we have 
\begin{equation}
(d^2-R^2)/2 = 1-\mathbf{e}^T\mathbf{A}_k^{-1}\mathbf{v} \le 0.
\end{equation}
Then
\begin{equation}
\alpha_{k+1} = \frac{1-\mathbf{e}^T\mathbf{A}_k^{-1}\mathbf{v}}{1-\mathbf{v}_k^T\mathbf{A}_k^{-1}\mathbf{v}}\le0. 
\end{equation}
\end{proof}

\thmref{notallsv} says that if we put a new data point $\mathbf{x}_i$ into an existing support vector set to form an expanded set and the $(k+1)$th position in the solution to the expanded system $\mathbf{A}_{k+1}\boldsymbol{\alpha}=\mathbf{e}$ is less than 0, then $\mathbf{x}_i$ is an interior point and thus can be ignored. Because we can permute the rows and columns in $\mathbf{A}_{k+1}^{-1}$, by \thmref{notallsv} if $\alpha_i\le0$ for $1\le i\le k$, we can take $\mathbf{x}_i$ out of the expanded set and solve the shrunken $k\times k$ linear system. We can continue shrinking the system until there are no negative entries in $\boldsymbol{\alpha}$; then a support vector set is obtained. We summarize this shrinking step in the next corollary.
\begin{corollary}\label{notallsvcor}
A data point $\mathbf{x}_{i}$ is an interior point if and only if $\mathbf{A}_{k+1}\boldsymbol{\alpha}=\mathbf{e}\Rightarrow\alpha_{i}\le0$ and the shrunken $k\times k$ linear system has a positive solution.
\end{corollary}

Finally, we state and prove an observation that relates the objective function value, the 1-norm of the \textit{unnormalized} $\boldsymbol{\alpha}$ vector, and the scoring threshold. The observation is substantial for implementing FISVDD. With it a lot of unnecessary computations can be saved. This observation can be also used to make sure that the objective function value in FISVDD is not larger than the objective function value obtained in the previous iteration so the FISVDD model is improved.
\begin{corollary}\label{alphanorm}
The objective function value in \eref{flexmat} with positive $\boldsymbol{\alpha}$, $\|\boldsymbol{\alpha}\|_1 = 1$, satisfies
\begin{equation}\label{alphanorm_eq1}
L = \frac{1}{\|\boldsymbol{\alpha}_0\|_1},
\end{equation}
where $\boldsymbol{\alpha}= \boldsymbol{\alpha}_0 / \|\boldsymbol{\alpha}_0\|_1$. Moreover, it holds that
\begin{equation}\label{alphanorm_eq2}
L = \sum_i\alpha_iK(\mathbf{z},\mathbf{x}_i),
\end{equation}
where the $\mathbf{x}_i$'s are the support vectors and $\mathbf{z}$ is any one of the support vectors.
\end{corollary}
\begin{proof}
To prove \eref{alphanorm_eq1}, note that by \thmref{allsv}, $\boldsymbol{\alpha}_0$ satisfies $\mathbf{A}\boldsymbol{\alpha}_0 = \mathbf{e}$. Then
\begin{equation}
\begin{split}
L&=\boldsymbol{\alpha}^T\mathbf{A}\boldsymbol{\alpha} = \frac{\boldsymbol{\alpha}_0^T}{\|\boldsymbol{\alpha}_0\|_1}\mathbf{A}\frac{\boldsymbol{\alpha}_0}{\|\boldsymbol{\alpha}_0\|_1} \\
&= \frac{\boldsymbol{\alpha}_0^T\mathbf{e}}{\|\boldsymbol{\alpha}_0\|_1^2} = \frac{\|\boldsymbol{\alpha}_0\|_1}{\|\boldsymbol{\alpha}_0\|_1^2}
= \frac{1}{\|\boldsymbol{\alpha}_0\|_1}.
\end{split}
\end{equation}
To prove \eref{alphanorm_eq2}, note that $\sum_i\alpha_iK(\mathbf{z},\mathbf{x}_i)$ is the first term of the right-hand side of \eref{dtb}. So proving \eref{alphanorm_eq2} is equivalent to proving
\begin{equation}
\sum_{i,j}\alpha_{i}\alpha _{j}K(\mathbf{x}_i,\mathbf{x}_j) = \sum_i\alpha_iK(\mathbf{z},\mathbf{x}_i),
\end{equation}
where $\mathbf{x}_i$, $\mathbf{x}_j$ are support vectors, and $\mathbf{z}$ is any one of the support vectors. The following equation can be derived:
\begin{equation}
\begin{split}
\sum_{i,j}\alpha_{i}\alpha _{j}K(\mathbf{x}_i,\mathbf{x}_j) 
&= \sum_j\alpha_{j}\Big(\sum_i\alpha_iK(\mathbf{x}_i,\mathbf{x}_j)\Big)\\
&=\Big(\sum_i\alpha_iK(\mathbf{z},\mathbf{x}_i)\Big)\Big(\sum_j\alpha_{j}\Big)\\
&=\sum_i\alpha_iK(\mathbf{z},\mathbf{x}_i).
\end{split}
\end{equation}
The second equality is derived from the fact that the term in parentheses is a constant for any support vector $\mathbf{x}_j$, and the third equality is derived from the fact that the sum of all $\alpha_i$'s is 1.
\end{proof}

\cref{alphanorm} shows a direct relationship between the objective function value, the 1-norm of the solution vector to the linear system $\mathbf{A}\boldsymbol{\alpha} = \mathbf{e}$, and the scoring threshold. The objective function value is a very important term of an SVDD model and can be requested by the user at any time. When the solution vector of the linear system is derived, the inverse of its 1-norm directly gives the objective function value, and the calculations in \eref{flexmat} are avoided. At the same time, $L$ is also the scoring threshold for the current model. Only the second term in \eref{dtb} needs to be computed when a new data point needs to be scored. The results from \cref{alphanorm} help make our FISVDD algorithm more efficient.

\section{Fast Incremental SVDD Learning Algorithm}\label{fisvdd}
We propose a fast incremental algorithm of SVDD (FISVDD). The central idea of FISVDD is to minimize the objective function (\ref{flexmat}) by quickly updating the inverse of similarity matrices in each iteration. Suppose that we begin with a support vector set $\mathbf{x}_1,\mathbf{x}_2,\ldots,\mathbf{x}_k$. When a new data point $\mathbf{x}_{k+1}$ comes in, by \thmref{allsv} the linear system $\mathbf{A}_{k+1}\boldsymbol{\alpha}=\mathbf{e}$ will have a positive solution if the $k+1$ data points form a new support vector set, and the normalized $\boldsymbol{\alpha}$ vector gives the $\alpha_i$'s. However, if at least one of the entries in the solution is negative, that indicates there is at least one interior point in the set. Then we are able to drop the negative $\alpha_i$ that has the largest $|d^2-R^2|$ magnitude and solve the shrunken $k\times k$ linear system. If the system has a positive solution, then we have found a support vector set. Otherwise, we can continue to drop the next negative $\alpha_i$ that has the largest $|d^2-R^2|$ magnitude and solve the $(k-1)\times (k-1)$ linear system, and so on. It is worth noting that if more than one variable is dropped from the system, the dropped data points should be re-scored against the new boundary to determine whether the KKT conditions are violated. If the KKT conditions are violated, then the system will expand again. We provide details below. 

\subsection{The FISVDD Algorithm}
The FISVDD algorithm is shown in \aref{fisvddalg}. It contains three parts of FISVDD: expanding (which is shown in \aref{expand}), shrinking (which is shown in \aref{shrink}), and bookkeeping. 

\subsubsection{Stage 1, Expanding}
When a new data point $\mathbf{x}_{k+1}$ comes in, it is scored to determine whether it falls in the interior. If so, it is immediately discarded. Otherwise, it is combined with existing support vectors to form an expanded set. The corresponding inverse matrix of the similarity matrix and its row sums are then updated by \lref{increinv}. If all row sums are positive, then $\mathbf{x}_{k+1}$ is another support vector and the normalized $\boldsymbol{\alpha}$ vector contains the updated $\alpha_i$'s. If $\alpha_{k+1}\le0$, then $\mathbf{x}_{k+1}$ is taken out of the expanded set and the support vector set returns to the previous set. If $\alpha_{k+1}>0$ but there is at least one $\alpha_{i}\le0$, then there is at least one interior point in the expanded set and the shrinking step is called. The expanding step is given in \aref{expand}.

\begin{algorithm}
   \centering
   \caption{Expand}\label{expand}
\begin{algorithmic}[1]
   \STATE {\bfseries Input:} {$\mathbf{x}_{k+1},\boldsymbol{\alpha},\mathrm{SV},\sigma,\mathbf{A}^{-1}$}
   \STATE $\mathbf{v} \gets K(\mathbf{x}_{k+1},\mathrm{SV},\sigma)$    
   \STATE $\mathbf{A}_{\mathrm{old}}^{-1}\gets\mathbf{A}^{-1}$
   \STATE $\mathbf{A}^{-1}\gets$ \eref{ak+1inv} 
   \STATE $\boldsymbol{\alpha}_{\mathrm{old}} \gets \boldsymbol{\alpha}$
   \STATE $\boldsymbol{\alpha}\gets$ row sums of $\mathbf{A}^{-1}$
   \IF {$\alpha_{k+1}\le 0$}
   		\STATE $\mathbf{A}^{-1}\gets\mathbf{A}_{\mathrm{old}}^{-1}$
   		\STATE $\boldsymbol{\alpha}\gets\boldsymbol{\alpha}_{\mathrm{old}}$
   \ELSE
   		\STATE SV $\gets$ SV + $\mathbf{x}_{k+1}$
   \ENDIF
\STATE \textbf{Return:} {$\boldsymbol{\alpha},\mathrm{SV},\mathbf{A}^{-1}$}
\end{algorithmic}
\end{algorithm}

\subsubsection{Stage 2, Shrinking}
If $\alpha_{k+1}>0$ but at least one $\alpha_i<0$, then at least one existing support vector in the support vector set has become an interior point. We need to identify and discard such vectors. By \cref{notallsvcor}, we can shrink the support vector set one vector at a time until a positive $\boldsymbol{\alpha}$ is obtained. It is possible that there are several negative entries in the $\boldsymbol{\alpha}$ vector, but after taking out one negative entry all other entries are positive. Hence, it is recommended to take out one vector at a time rather than taking out several vectors. Moreover, taking out several vectors at once slows the algorithm because then we need to calculate the inverse of matrices whose rank is larger than 1. Although there is no certain way of choosing which vector to remove first, in FISVDD we choose the negative $\alpha_i$ that has the largest magnitude. From \eref{alphak+1} and \eref{dist2-r2} and permuting columns and rows in $\mathbf{A}_{k+1}$, we have
\begin{equation}\label{shrink1}
\alpha_{k+1} = \frac{d^2-R^2}{2(1-\mathbf{v}^T\mathbf{A}_k^{-1}\mathbf{v})},
\end{equation}
where $\alpha_{k+1}$ is the $\alpha_i$ of interest permuted to the $(k+1)$th position. It can be seen from \eref{shrink1} that if the denominators of the data points that have negative $\alpha_i$'s are close, then a data point that has a larger $|\alpha_i|$ tends to have a larger $|d^2-R^2|$, which means it lies farther from the boundary. Intuitively, a data point farther from the boundary is more likely to be a true interior point. Although not guaranteed, the data point farthest from the boundary is typically the one we want to remove first.

\begin{algorithm}
\centering
\caption{Shrink}\label{shrink}
\begin{algorithmic}[1]
\STATE {\bfseries Input:} {$\boldsymbol{\alpha},\mathrm{SV},\mathbf{A}^{-1},\mathrm{Backup}$}
\STATE $\mathrm{flag} \gets 1$
\WHILE {$\mathrm{flag} = 1$}
\STATE $p \gets \arg\min\boldsymbol{\alpha}$
\STATE $\mathrm{Backup} \gets \mathrm{Backup} + \mathbf{x}_p$
\STATE $\mathrm{SV} \gets \mathrm{SV} - \mathbf{x}_p$
\STATE $\mathbf{A}^{-1}\gets$\eref{ak-1inv} 
\STATE $\boldsymbol{\alpha}\gets$ row sums of $\mathbf{A}^{-1}$
\IF {$\min\boldsymbol{\alpha}>0$}
        \STATE $\mathrm{flag} \gets 0$
\ENDIF
\ENDWHILE
\STATE \textbf{Return} {$\boldsymbol{\alpha},\mathrm{SV},\mathbf{A}^{-1},\mathrm{Backup}$}
\end{algorithmic}
\end{algorithm}

\subsubsection{Bookkeeping}
When the shrinking algorithm is performed, some of the previous support vectors are taken out of the support vector set if they have negative $\alpha_i$'s. However, having a negative $\alpha_i$ in the middle of a shrinking process does not rule a support vector out from the final set. A data point is considered to be an interior point only if it satisfies $(d^2-R^2)<0$ when scored with the final support vector set. Therefore, it is necessary to recheck whether the data points taken out of the support vector set are truly interior points. In FISVDD, we build a backup set when the shrinking stage begins. When a data point is taken out of the support vector set, it is put into the backup set. Then the inverse matrix is ``downdated'' with \eref{ak-1inv} and its row sums are calculated. The shrinking continues until there are no negative entries in the $\boldsymbol{\alpha}$ vector. The backup set keeps growing as the linear system shrinks. When there are no negative values in $\boldsymbol{\alpha}$, we have found a support vector set, although it might not be the final one. Then the data points in the backup set are scored with the support vector set one by one in a first in, first out order. To increase the algorithm's efficiency, the backup set is scanned only once. If $(d^2-R^2)>0$ for a data point, then the expanding algorithm is called again, and the data point is removed from the backup set and placed back into the support vector set. The expanding finishes when all data points in the backup set have $(d^2-R^2)\le0$. Although the same check can be performed on all prior data, doing so would cost too much memory and the gains are far less significant. So the backup set is emptied when each new data point arrives.

For completeness, we add a check to the unnormalized $\boldsymbol{\alpha}$ vector to make sure that the result in each iteration is improved from the previous iteration. By \cref{alphanorm}, the result is improved if the 1-norm of the unnormalized $\boldsymbol{\alpha}$ vector increases. At the end of each iteration, this norm is compared with the norm in the previous iteration. If the norm decreases, then the result from the previous iteration is restored. None of our experiments have ever violated this condition.

\begin{algorithm}
\caption{Fast Incremental Support Vector Data Description (FISVDD)}\label{fisvddalg}
\begin{algorithmic}[1]
\STATE {\bfseries Input:} {$\mathrm{Initialize}(\boldsymbol{\alpha},\mathrm{SV},\mathbf{A}^{-1},\sigma$)}
\FOR{$i \gets 1,\, n$}
\STATE $Q\gets$ \eref{dtb}
\IF {$Q \le 0$}
        \STATE pass
\ELSE
        \STATE $\boldsymbol{\alpha},\mathrm{SV},\mathbf{A}^{-1}\gets\mathrm{Expand}(\mathbf{x}_{k+1},\boldsymbol{\alpha},\mathrm{SV},\sigma,\mathbf{A}^{-1})$
        \IF {$\min\boldsymbol{\alpha}<0$}
                \STATE $\mathrm{Backup}\gets$ Empty set
                \STATE $\boldsymbol{\alpha},\mathrm{SV}, \mathbf{A}^{-1},\mathrm{Backup}\gets\mathrm{Shrink}(\boldsymbol{\alpha},\mathrm{SV},\mathbf{A}^{-1},\mathrm{Backup})$
                \IF {$\mathbf{card}(\mathrm{Backup})>1$}
                        \FOR{$j \gets 1,\, \mathbf{card}(\mathrm{Backup})$} 
                                \STATE $Q\gets$\eref{dtb}
                                \IF {$Q>0$}
                                        \STATE $\boldsymbol{\alpha},\mathrm{SV},\mathbf{A}^{-1}\gets\mathrm{Expand}(\mathrm{Backup}_j,\boldsymbol{\alpha},\mathrm{SV},\sigma,\mathbf{A}^{-1})$
                                \ENDIF
                        \ENDFOR
                \ENDIF
        \ENDIF
\STATE $\boldsymbol{\alpha}\gets\boldsymbol{\alpha}/\|\boldsymbol{\alpha}\|_1$
\ENDIF
\ENDFOR
\end{algorithmic}
\end{algorithm}

To summarize, FISVDD is fast and computationally efficient because the algorithm ignores interior points and is built solely on matrix manipulations. First, FISVDD tries to obtain the optimal solution in each iteration without using the interior points, similar to the idea mentioned in \citet{syed1999incremental}. Results from many experiments show that if a proper Gaussian bandwidth is chosen, then the number of support vectors should be far smaller than the total number of observations. FISVDD takes advantage of this fact by calculating only the similarities between the new data points and the support vectors.

Secondly, it can be seen from \aref{fisvddalg} that FISVDD is based only on matrix manipulation. Matrix inverse updating steps are the core of FISVDD, which lets the system itself choose which data points to move between support vector sets and interior point sets. Sometimes the choice of the system might not be optimal, but the existence of backup sets allows the system to correct itself and removes a significant number of calculations.

\section{Implementation Details}\label{misc}
In this section we discuss several important details for implementing FISVDD.

\subsection{Initialization}
A key advantage of FISVDD is that the similarity matrix $\mathbf{A}$ is directly calculated only at initialization. As stated in \sref{fisvdd}, each iteration calculates only the similarities between a new data point and the existing support vectors. These are used to update the inverse of the similarity matrix; the similarity matrix is calculated only at initialization. Once the burn-in data points are selected, their similarity matrix $\mathbf{A}$ and its inverse $\mathbf{A}^{-1}$ are calculated. After the row sums of $\mathbf{A}^{-1}$ are calculated, the shrinking step in \aref{shrink} is used to pick out the interior points. Then the vector that contains the normalized row sums of $\mathbf{A}^{-1}$ is the initial $\boldsymbol{\alpha}$.

\subsection{Memory}
For any online method, it is important to make sure that both of the following conditions hold:
\begin{itemize}
\item The complexity in each step is small.
\item Memory usage will never expand out of control even for very large data.
\end{itemize}  
For FISVDD, the two challenges are handled smoothly. The first part is easy to see: The key parts in the algorithm (expanding and shrinking the linear systems) require only $O(k^2)$ multiplications each time, where $k$ is the number of support vectors. In addition, $k$ should be far less than the total number of the whole data set if a proper Gaussian kernel bandwidth $\sigma$ is chosen.  

For the second part, the number of support vectors can indeed grow large with streaming data. To avoid the potential threat of memory expanding out of control, we set a parameter, $M$, for the maximal number of support vectors, where $M$ depends on availability of memory. When $M$ is reached, the number of support vectors will not grow large. If a new data point $\mathbf{x}_{k+1}$ satisfies $d^2>R^2$, then one of the three situations will occur:
\begin{itemize}
\item $\alpha_{k+1}>0$ but at least one of the $\alpha_i$'s is less than or equal to 0. In this case, the algorithm runs normally to select the interior points.
\item All $\alpha_{i}$'s are greater than 0, but $\alpha_{k+1}$ is the smallest among all $\alpha_i$'s. In this case, $\alpha_{k+1}$ is discarded.
\item All $\alpha_{i}$'s are greater than 0, and $\alpha_{k+1}$ is not the smallest among all $\alpha_i$'s. In this case, the support vector that has the smallest $\alpha_{i}$ is replaced by $\mathbf{x}_{k+1}$, and the new $\alpha_i$'s are updated.
\end{itemize}

By handling these three cases, the number of support vectors will not exceed $M$, and the memory usage in each step is controlled.

\subsection{Outliers and Close Points}\label{epsilon}
Until now, our analysis focused primarily on describing the boundary of the streaming data. Another important feature of SVDD is that it finds outliers in the data so that further investigations can be taken. In \cite{laskov2006incremental} and \cite{scheinberg2006efficient}, data points are classified as outliers based on $\alpha_i$ values. FISVDD assumes that outliers are far from normal data and hence do not influence the support vectors and the $\alpha_i$'s. In addition, we assume that the boundary that is determined by the support vectors is robust to outliers. Note that if a data point is far from the support vectors, the $\mathbf{v}$ vector in \eref{vvec} should be close to a zero vector, which indicates that the largest value in $\mathbf{v}$ should be close to 0. In FISVDD, a data point $\mathbf{z}$ is classified as an outlier if it satisfies the following condition for a preselected parameter $\epsilon_1>0$:
\begin{equation}\label{outlierdef}
\max{\mathbf{v}} < \epsilon_1.
\end{equation}
If $\mathbf{z}$ is classified as an outlier, then it is passed to further investigation, and no $\alpha$ value is assigned to it.

Another special case we have to consider is a new data point that is very close to one of the existing support vectors. Although in practice it is rare that a new data point is exactly the same as an existing support vector, it is possible that they are very close to each other. In this case, the similarity matrix $\mathbf{A}$ will be ill-conditioned and $\mathbf{A}^{-1}$ might be not accurate. We can avoid this situation by also looking at the maximal entry value in $\mathbf{v}$. If a new data point is very close to one of the support vectors, then the maximal entry value in $\mathbf{v}$ will be close to 1. In FISVDD, a point is discarded if it satisfies the following condition for a preselected parameter $\epsilon_2>0$:
\begin{equation}\label{outlierdef}
\max{\mathbf{v}} > 1-\epsilon_2.
\end{equation}

Finally, note that these preprocessing steps can help prevent unnecessary calculations if the Gaussian kernel bandwidth $\sigma$ is not a proper bandwidth. If $\sigma$ is too small, then every data point tends to be a support vector and the similarity between every pair of data points is close to 0. If $\sigma$ is too large, then the similarity between every pair of data points is close to 1. Introducing $\epsilon_1$ and $\epsilon_2$ can prevent these cases.

\section{Experiments}\label{experiments}
We examined the performance of FISVDD with four real data sets: shuttle data \citep{Lichman:2013}, mammography data \citep{woods1993comparative}, forest cover (ForestType) data \citep{Rayana:2016}, and the SMTP subset of KDD Cup 99 data \citep{Rayana:2016}. The purpose of our experiments is to show that compared to the incremental SVM method (which can achieve global optimal solutions), the FISVDD method does not lose much in either objective function value or outlier detection accuracy while it demonstrates significant gains in efficiency. Our experiments used 4/5 of the normal data, randomly chosen, for training. The remaining normal data and the outliers together form the testing sets. All duplicates in the data sets are removed beforehand. Proper Gaussian bandwidths are selected by using fivefold cross validation, although selecting a proper Gaussian bandwidth is beyond the scope of this paper.  SAS/IML\textsuperscript{\tiny\textregistered} software is used in performing the experiments. In this paper, we compare FISVDD with the one-class incremental SVM method \citep{laskov2006incremental}, a well-known technique for performing global optimal SVDD. For each method, the following quantities are measured in \tref{table_compare}:

\begin{itemize}
\item Time: The time used to learn the SVDD model.
\item Objective function value (OFV): The objective function values that were obtained with \eref{flexmat} after each iteration.
\item Number of support vectors (\#sv): The number of support vectors when the training phase is finished. This number is related to the efficiency of the testing phase. When more support vectors exist, more calculations are required in testing. 
\end{itemize}

The time consumed by the incremental SVM method with interior points discarded after each iteration is listed in parentheses. \tref{table_compare} also lists the settings for the experiments, including Gaussian bandwidth (Sigma), number of training observations (\#Train obs), number of testing observations (\#Test obs), and number of variables (\#Var).

\begin{table}[h]
\centering
\caption{Experimental Results of FISVDD and Incremental SVM on Different Data Sets}\label{table_compare}
\resizebox{\columnwidth}{!}{
\begin{tabular}{l | c c c c c c c c}
\toprule
Data & Sigma & Method &$\#$Train obs &$\#$Test obs &\#Var & OFV & Time (s) &$\#$sv\\
\midrule
Shuttle				&5.5	&FISVDD	&36469	&21531	&9		&$\expnumber{1.7378}{-3}$		&251.01		&1736	\\
						&		&Inc. SVM	&			&			&		&$\expnumber{1.7369}{-3}$		&22923.57	&1926	\\
						&		&				&			&			&		&											&(312.65)	&		 	\\
CoverType			&470	&FISVDD	&226641	&59407	&10	&$\expnumber{1.14158}{-2}$	&19.47		&432	 	\\
						&		&Inc. SVM	&			&			&		&$\expnumber{1.14155}{-2}$	&12954.81	&470  	\\
						&		&				&			&			&		&											&(29.45)		&		 	\\
Mammography		&0.8	&FISVDD	&6076	&1773	&6		&$\expnumber{9.8134}{-3}$		&1.19 		&317	 	\\
						&		&Inc. SVM	&			&			&		&$\expnumber{9.8008}{-3}$		&67.01   	&317  	\\
						&		&				&			&			&		&											&(1.58)		&		 	\\
Smtp					&6		&FISVDD	&56967	&14263	&3		&0.393									&0.27			&5		 	\\
						&		&Inc. SVM	&			&			&		&0.393									&2.49			&5		  	\\
						&		&				&			&			&		&											&(0.38)		&		 	\\
\bottomrule
\end{tabular}
}
\end{table}

\tref{table_compare} shows that for the same Gaussian bandwidth, the FISVDD method is much faster than the incremental SVM method, with only a tiny sacrifice in the objective function value. Because incremental SVM achieves global optimal solutions, the solutions provided by FISVDD are very close to the global optimal solutions. Even with interior points discarded after each iteration, FISVDD is faster than incremental SVM for the data sets in our experiments.  As explained in \sref{fisvdd}, FISVDD is faster because it is based solely on matrix manipulation and thus many calculations are saved. 

\fref{pics} shows plots of the F-1 measure \citep{tan2007introduction} of the accuracy of FISVDD and incremental SVM with different training sizes. The plots show that by discarding interior points at the end of each iteration, there is almost no loss in the quality of outlier detection.

\begin{figure}
\centering
\begin{subfigure}{.5\columnwidth}
  \centering
  \includegraphics[width=\columnwidth]{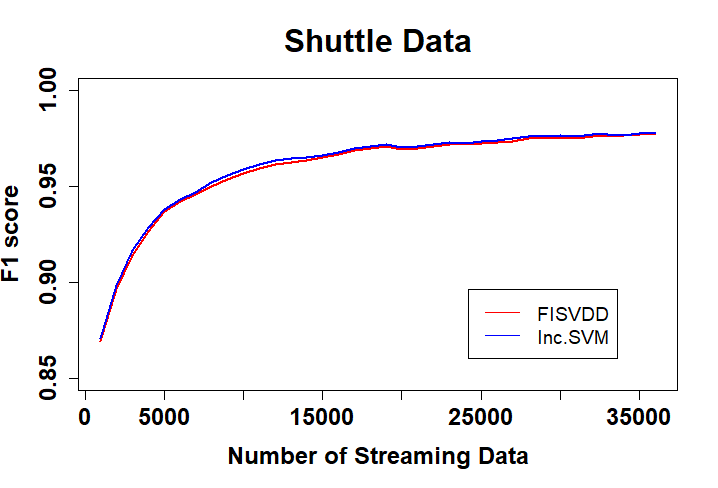}
 \caption{Shuttle Data}
  \label{fig:sub1}
\end{subfigure}%
\begin{subfigure}{.5\columnwidth}
  \centering
  \includegraphics[width=\columnwidth]{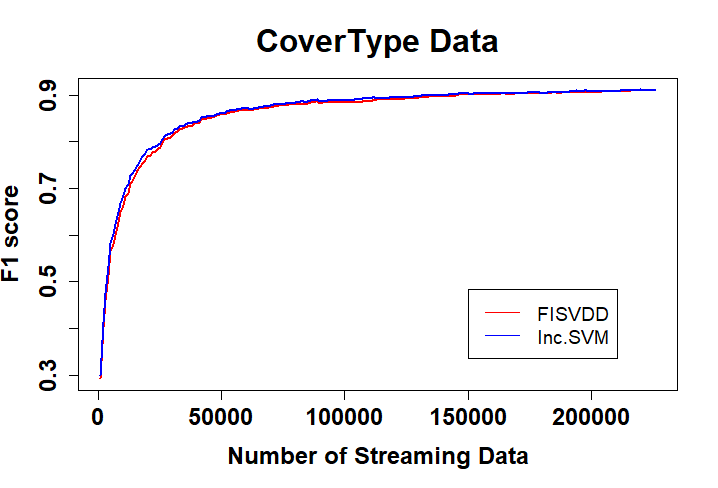}
  \caption{CoverType Data}
  \label{fig:sub2}
\end{subfigure}\\[1ex]
\begin{subfigure}{.5\columnwidth}
  \centering
  \includegraphics[width=\columnwidth]{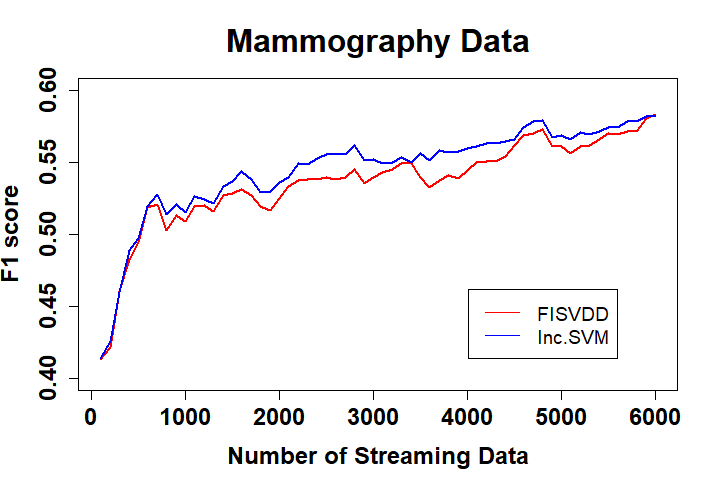}
  \caption{Mammography Data}
  \label{fig:sub3}
\end{subfigure}%
\begin{subfigure}{.5\columnwidth}
  \centering
  \includegraphics[width=\columnwidth]{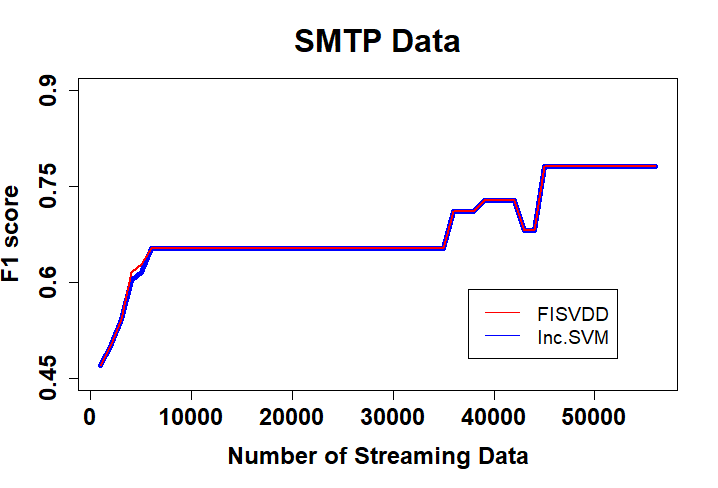}
  \caption{SMTP Data}
  \label{fig:sub4}
\end{subfigure}
\caption{F-1 Measure for Different Data Sets}
\label{pics}
\end{figure}

\section{Conclusion}\label{conclusion}
This paper introduces a fast incremental SVDD learning algorithm (FISVDD), which is more efficient than existing SVDD algorithms. In each iteration, FISVDD considers only the incoming data point and the support vectors that were determined in the previous iteration. The essential calculations of FISVDD are contributed from incremental and decremental updates of a similar matrix inverse $\mathbf{A}^{-1}$. This algorithm builds on an observation that is natural in SVDD models but has not been fully utilized by existing SVDD algorithms: that all support vectors on the boundary have the same distance to the center of sphere in a higher-dimensional feature space as mapped by the Gaussian kernel function. FISVDD uses the signs of entries in the row sums of $\mathbf{A}^{-1}$ to determine the interior points and support vectors and uses their magnitudes to determine the Lagrange multiplier $\alpha_i$ for each support vector. Experimental results indicate that FISVDD gains much efficiency with almost no loss in accuracy and objective function value.

\section*{Acknowledgement}
We would like to thank Anne Baxter, Maria Jahja, and Cong Meng for their help in this paper. We would also like to thank Minghui Liu, Joshua Griffin, Yuwei Liao, and Seunghyun Kong for discussions that are related to SVDD.

\bibliographystyle{plainnat}
\bibliography{mybib1}
\end{document}